\documentclass{article}

\PassOptionsToPackage{numbers}{natbib}

\usepackage[final]{neurips_2019}

\usepackage[utf8]{inputenc} 
\usepackage[T1]{fontenc}    
\usepackage{hyperref}       
\usepackage{url}            
\usepackage{booktabs}       
\usepackage{amsfonts}       
\usepackage{nicefrac}       
\usepackage{microtype}      

\usepackage[pdftex]{graphicx}
\usepackage{floatrow}
\usepackage{booktabs}       
\usepackage[noend]{algorithmic}
\usepackage{algorithm}
\usepackage{caption}
\usepackage{subcaption}

\usepackage{xspace}

\usepackage{footnote}
\makesavenoteenv{algorithm}

\newlength{\commentindent}
\makeatletter
\renewcommand{\algorithmiccomment}[1]{\unskip\hfill\makebox[\commentindent][l]{//~#1}\par}
\LetLtxMacro{\oldalgorithmic}{\algorithmic}
\renewcommand{\algorithmic}[1][0]{%
  \oldalgorithmic[#1]%
  \renewcommand{\ALC@com}[1]{%
    \ifnum\pdfstrcmp{##1}{default}=0\else\algorithmiccomment{##1}\fi}%
}
\makeatother

\usepackage{amsmath}
\usepackage{amssymb}
\usepackage{amsthm}
\usepackage{comment}

\usepackage{enumitem}
\setlist{itemsep=-4pt, topsep=0pt}  

\newcommand\refine{\mathit{Refine}}
\newcommand\Zub{Z^{\mathrm{UB}}}

\newcommand\Zubunder{\underline{Z}^{\mathrm{UB}}}
\newcommand\method{\textsc{AdaPart}\xspace} 

\newtheorem{prop}{Proposition}

\newtheorem{definition}{Definition}

\usepackage{etoolbox,refcount}
\usepackage{multicol}
\newcounter{countitems}
\newcounter{nextitemizecount}
\newcommand{\setupcountitems}{%
  \stepcounter{nextitemizecount}%
  \setcounter{countitems}{0}%
  \preto\item{\stepcounter{countitems}}%
}
\makeatletter
\newcommand{\computecountitems}{%
  \edef\@currentlabel{\number\c@countitems}%
  \label{countitems@\number\numexpr\value{nextitemizecount}-1\relax}%
}
\newcommand{\nextitemizecount}{%
  \getrefnumber{countitems@\number\c@nextitemizecount}%
}
\newcommand{\previtemizecount}{%
  \getrefnumber{countitems@\number\numexpr\value{nextitemizecount}-1\relax}%
}
\makeatother 
{\end{enumerate}%
\unskip\computecountitems\ifnumcomp{\previtemizecount}{>}{3}{\end{multicols}}{}}

\usepackage{chngcntr}
\usepackage{apptools}
\AtAppendix{\counterwithin{algorithm}{section}}
\AtAppendix{\counterwithin{prop}{section}}

\usepackage{xcolor} 
\newcommand\jonathan[1]{\textcolor{blue}{[JK: #1]}}

\newcommand\tri[1]{\textcolor{green}{[Tri: #1]}}
\newcommand\se[1]{\textcolor{red}{[SE: #1]}}
\newcommand\s[1]{\textcolor{red}{[SE: #1]}}
\newcommand\as[1]{\textcolor{magenta}{[AS: #1]}}

\renewcommand\jonathan[1]{} 
\renewcommand\s[1]{} 
\renewcommand\tri[1]{} 
\renewcommand\se[1]{} 
\renewcommand\as[1]{} 

\title{Approximating the Permanent by\\ Sampling from Adaptive Partitions}

\author{Jonathan Kuck\textsuperscript{1}, \ \  Tri Dao\textsuperscript{1},\ \ Hamid Rezatofighi\textsuperscript{1},\ \  Ashish Sabharwal\textsuperscript{2}, \normalfont{and} \bf{Stefano Ermon}\textsuperscript{1}\\
  \textsuperscript{1}Stanford University\ \ \ \ \
  \textsuperscript{2}Allen Institute for Artificial Intelligence\\
{\tt\small \{kuck,trid,hamidrt,ermon\}@stanford.edu, ashishs@allenai.org}
}

\begin{document}

\maketitle

\begin{abstract}
Computing the permanent of a non-negative matrix is a core problem with practical applications ranging from target tracking to statistical thermodynamics. However, this problem is also \#P-complete, which leaves little hope for finding an exact solution that can be computed efficiently.  While the problem admits a fully polynomial randomized approximation scheme, this method has seen little use because it is both inefficient in practice and difficult to implement. We present \method, a simple and efficient method for drawing exact samples from an unnormalized distribution. Using \method, we show how to construct tight bounds on the permanent which hold with high probability, with guaranteed polynomial runtime for dense matrices. We find that \method can provide empirical speedups exceeding 25x over prior sampling methods on matrices that are challenging for variational based approaches. Finally, in the context of multi-target tracking, exact sampling from the distribution defined by the matrix permanent allows us to use the optimal proposal distribution during particle filtering. Using \method, we show that this leads to improved tracking performance using an order of magnitude fewer samples.
\end{abstract}

\section{Introduction}

The permanent of a square, non-negative matrix $A$ is a quantity with natural graph theoretic interpretations.  If $A$ is interpreted as the adjacency matrix of a directed graph, the permanent corresponds to the sum of weights of its cycle covers. If the graph is bipartite, it corresponds to the sum of weights of its perfect matchings. The permanent has many applications in computer science and beyond.  In target tracking applications \cite{uhlmann2004matrix,morelande2009joint,oh2009markov,hamid2015joint}, it is used to calculate the marginal probability of  measurements-target associations. In general computer science, it is widely used in graph theory and network science. The permanent also arises in statistical thermodynamics \cite{beichl1999approximating}.

Unfortunately, computing the permanent of a matrix is believed to be intractable in the worst-case, as the problem has been formally shown to be \#P-complete~\cite{valiant1979complexity}.
Surprisingly, a fully polynomial randomized approximation scheme (FPRAS) exists, meaning that it is theoretically possible to accurately approximate the permanent in polynomial time. However, this algorithm is not practical: it is difficult to implement and it scales as $O(n^7 \log^4 n)$.  Ignoring coefficients, this is no better than exact calculation until matrices of size $40 \text{x} 40$, which takes days to compute on a modern laptop.


 		

The problems of sampling from an unnormalized distribution and calculating the distribution's normalization constant (or partition function) are closely related and interreducible. 
An efficient solution to one problem leads to an efficient solution to the other~\cite{jerrum1986random,jerrum1996markov}.  Computing the permanent of a matrix is a special instance of computing the partition function of an unnormalized distribution~\cite{wainwright2008graphical}.  In this case the distribution is over $n!$ permutations, the matrix defines a weight for each permutation, and the permanent is the sum of these weights.

\subsection{Contributions} First, we present \method, a novel method for \textit{drawing exact samples from an unnormalized distribution using any algorithm that upper bounds its partition function.}  We use these samples to estimate and bound the partition function with high probability.  This is a generalization of prior work~\cite{huber2006exact,law2009approximately}, which showed that a specific bound on the matrix permanent nests, or satisfies a Matryoshka doll like property where the bound recursively fits within itself, for a fixed partitioning of the state space.  Our novelty lies in adaptively choosing a partitioning of the state space, which (a) is suited to the particular distribution under consideration, and (b) allows us to use any upper bound or combination of bounds on the partition function, rather than one that can be proven \emph{a priori} to nest according to a fixed partitioning.

Second, we provide a complete instantiation of \method for sampling permutations with weights defined by a matrix, and correspondingly computing the permanent of that matrix. To this end, we identify and use an upper bound on the permanent with several desirable properties, including being computable in polynomial time and being tighter than the best known bound that provably nests.

Third, we empirically demonstrate that \method is both computationally efficient and practical for approximating the permanent of a variety of matrices, both randomly generated and from real world applications.  We find that \method can be over 25x faster compared to prior work on sampling from and approximating the permanent.  In the context of multi-target tracking, \method facilitates sampling from the optimal proposal distribution during particle filtering, which improves multi-target tracking performance while reducing the number of samples by an order of magnitude.

\section{Background}

The \emph{permanent} of an $n \times n$ non-negative matrix $A$ is defined as $\text{per}(A) = \sum_{\sigma \in S_n} \prod_{j=1}^n A(j, \sigma(j))$, where the sum is over all permutations $\sigma$ of $\{1,2,\dots,n\}$ and $S_n$ denotes the corresponding symmetric group.  Let us define the weight function, or unnormalized probability, of a permutation as $w(\sigma) = \prod_{j=1}^n A(j, \sigma(j))$.  The permanent can then be written as $\text{per}(A) = \sum_{\sigma \in S_n} w(\sigma)$, which is the partition function (normalization constant) of $w$, also denoted $Z_w$.

We are interested in sampling from the corresponding probability distribution over permutations $p(\sigma) = \frac{w(\sigma)}{\sum_{\sigma' \in S_n} w(\sigma')}$, or more generally from any unnormalized
distribution where the exact partition function is unknown. Instead, we will assume access to a function that \emph{upper bounds} the partition function, for instance an upper bound on the permanent. By verifying (at runtime) that this upper bound satisfies a natural `nesting' property w.r.t.\ a partition of the permutations, we will be able to guarantee exact samples from the underlying distribution.  Note that verification is critical since the `nesting' property does not hold for upper bounds in general.

In the next few sections, we will consider the general case of any non-negative weight function $w$ over $N$ states (i.e., $w : \mathcal{S} \rightarrow \mathbb{R}_{\geq 0}, |\mathcal{S}| = N$) and its partition function $Z_w$, rather than specifically discussing weighted permutations of a matrix and its permanent.  This is to simplify the discussion and present it in a general form. We will return to the specific case of the permanent later on.

\subsection{Nesting Bounds}

\citet{huber2006exact} and \citet{law2009approximately} have noted that upper bounds on the partition function that `nest' can be used to draw exact samples from a distribution defined by an arbitrary, non-negative weight function. For their method to work, the upper bound must nest according to some fixed partitioning $\mathcal{T}$ of the weight function's state space, as formalized in Definition~\ref{def:partition_tree}.  In Definition~\ref{nesting_UB}, we state the properties that must hold for an upper bound to `nest' according to the partitioning $\mathcal{T}$.

\begin{definition}[Partition Tree]
\label{def:partition_tree}
Let $\mathcal{S}$ denote a finite state space. A \emph{partition tree} $\mathcal{T}$ for $\mathcal{S}$ is a tree where each node is associated with a non-empty subset of $\mathcal{S}$ such that:
\begin{enumerate}
    \item The root of $\mathcal{T}$ is associated with $\mathcal{S}$.
    \item If $\mathcal{S} = \{a\}$, the tree $\{a\}$ formed by a single node is a partition tree for $\mathcal{S}$.
    \item Let $v_1, \cdots, v_k$ be the children of the root node of $\mathcal{T}$, and $S_1, \cdots, S_k$ be their associated subsets of $\mathcal{S}$. $\mathcal{T}$ is a partition tree if $S_i, S_j$ are pairwise disjoint, $\cup_i S_i = \mathcal{S}$, and for each $\ell$ the subtree rooted at $v_\ell$ is a partition tree for $S_\ell$.
\end{enumerate}
\end{definition}

\begin{definition}[Nesting Bounds]
\label{nesting_UB}
Let $w:\mathcal{S} \rightarrow \mathbb{R}_{\geq 0}$ be a non-negative weight function with partition function $Z_w$. 
%
%
%
%
%
%
Let $\mathcal{T}$ be a partition tree for $\mathcal{S}$ and let $\mathcal{S}_\mathcal{T}$ be the set containing the subsets of $\mathcal{S}$ associated with each node in $\mathcal{T}$.
The function $\Zub_w(S) : \mathcal{S}_\mathcal{T} \rightarrow \mathbb{R}_{\geq 0}$ is a \emph{nesting upper bound} for $Z_w$ with respect to $\mathcal{T}$ if:

\begin{enumerate} 
    \item The bound is tight for all single element sets: $\Zub_w(\{i\}) = w(i)$ for all $i \in \mathcal{S}$.\footnote{This requirement can be relaxed by defining a new upper bounding function that returns $w(i)$ for single element sets and the upper bound which violated this condition for multi-element sets.}
    \item The bound `nests' at every internal node $v$ in $\mathcal{T}$.  Let $S$ be the subset of $\mathcal{S}$ associated with $v$.  Let $S_1, \cdots, S_k$ be the subsets associated with the children of $v$ in $\mathcal{T}$. Then the bound `nests' at $v$ if:
    \begin{align}
    \label{eq:nestingbound}
      \sum_{\ell=1}^k \Zub_w (S_\ell) \leq \Zub_w (S).
    \end{align}    
\end{enumerate}
\end{definition}

\subsection{Rejection Sampling with a Fixed Partition}\label{sec:fixed_nesting}

Setting aside the practical difficulty of finding such a bound and partition, suppose we are \emph{given} a fixed partition tree $\mathcal{T}$ and a guarantee that $\Zub_w$ nests according to $\mathcal{T}$. Under these conditions, \citet{law2009approximately} proposed a rejection sampling method to perfectly sample an element, $i \sim \frac{w(i)}{\sum_{j \in \mathcal{S}} w(j)}$, from the normalized weight function (see Algorithm~\ref{alg:nesting_sample} in the Appendix).  Algorithm~\ref{alg:nesting_sample} takes the form of a rejection sampler whose proposal distribution matches the true distribution precisely---except for the addition of \emph{slack} elements with joint probability mass equal to $\Zub_w(\mathcal{S}) - Z_w$.  The algorithm recursively samples a partition of the state space until the sampled partition contains a single element or slack is sampled.  Samples of slack are rejected and the 
procedure is repeated until 
a valid single element is returned.  

According to Proposition~\ref{prop:rejection_sampling_correctness} (see Appendix), 
Algorithm~\ref{alg:nesting_sample} yields exact samples from the desired target distribution. Since it performs rejection sampling using $\Zub_w(S)$ to construct a proposal, its efficiency depends on how close the proposal distribution is to the target distribution.  In our case, this is governed by two factors: (a) the tightness of the (nesting) upper bound, $\Zub_w(S)$, and (b) the tree $\mathcal{T}$ used to partition the state space (in particular, it is desirable for every node in the tree to have a small number of children).

In what follows, we show how to substantially improve upon Algorithm~\ref{alg:nesting_sample} by utilizing tighter bounds (even if they don't nest \emph{a priori}) and iteratively checking for the nesting condition at runtime until it holds.

\begin{algorithm}[t!]
  \caption{\method: Sample from a Weight Function using Adaptive Partitioning}
\footnotesize
  \label{alg:nesting_sample_partition}
  \begin{flushleft}
  \textbf{Inputs:}
  \end{flushleft}

\begin{enumerate}
    \item Non-empty state space, $\mathcal{S}$
    \item Unnormalized weight function, $w: \mathcal{S} \rightarrow \mathbb{R}_{\geq 0}$
    \item Family of upper bounds for $w$,
    $\Zub_w(S):\mathcal{D} \subseteq 2^{\mathcal{S}} \rightarrow \mathbb{R}_{\geq 0}$
    \item Refinement function, $\refine: \mathcal{P} \rightarrow 2^{\mathcal{P}}$, where $\mathcal{P}$ is the set of all partitions of $\mathcal{S}$
\end{enumerate}
  \begin{flushleft}
  \textbf{Output:} A sample $i \in \mathcal{S}$ distributed as $i \sim \frac{w(i)}{\sum_{i \in \mathcal{S}} w(i)}$.
  \end{flushleft}

 \begin{algorithmic}
 \STATE \algorithmicif\ {$\mathcal{S}=\{a\}$} \algorithmicthen\ {Return $a$}
 
 \STATE $P = \{\mathcal{S}\}$ \ \ ; \ \ $\mathit{ub} \leftarrow \Zub_w(\mathcal{S})$
 \REPEAT

   \STATE Choose a subset $S \in P$ to refine:  $\{\{S^i_1, \cdots, S^i_{\ell_i}\}\}_{i=1}^K \leftarrow \refine(S)$
   \FORALL {$i \in \{1, \cdots, K\}$}
     \STATE $\mathit{ub}_i \leftarrow \sum_{j=1}^{\ell_i} \Zub_w(S^i_j)$
   \ENDFOR
   \STATE $j \leftarrow \arg \min_i \mathit{ub}_i$ \ ; \ $P \leftarrow (P \setminus \{S\}) \cup \{S^j_1, \cdots, S^j_{\ell_j}\}$ \ ; \ $\mathit{ub} \leftarrow \mathit{ub} - \Zub_w(S) + \mathit{ub}_j$
 \UNTIL {$\mathit{ub} \leq \Zub_w(\mathcal{S})$}
 \STATE Sample a subset $S_i \in P$ with prob.\
 $\frac{\Zub_w(S_i)}{\Zub_w(\mathcal{S})}$, or sample $\mathit{slack}$ with prob.~$1-\frac{\mathit{ub}}{\Zub_w(\mathcal{S})}$
  \STATE \algorithmicif\ {$S_m \in P$ is sampled} \algorithmicthen\ {Recursively call \method($S_m,w,\Zub_w,\refine$)}
 \STATE \algorithmicelse\ {Reject $\mathit{slack}$ and restart with the call to \method($\mathcal{S},w,\Zub_w,\refine$)}
\normalsize 
 \end{algorithmic}
\end{algorithm}

\section{Adaptive Partitioning}
A key limitation of using the approach in Algorithm~\ref{alg:nesting_sample} is that it is painstaking to prove \emph{a priori} that an upper bound nests for a yet unknown weight function with respect to a complete, fixed partition tree. Indeed, a key contribution of prior work \cite{huber2006exact,law2009approximately} has been to provide a proof that a particular upper bound nests for any weight function $w:\{1,\dots, N\} \rightarrow \mathbb{R}_{\geq 0}$ according to a fixed partition tree whose nodes all have a small number of children.

In contrast, we observe that it is nearly trivial to empirically verify \emph{a posteriori} whether an upper bound respects the nesting property for a particular weight function for a particular partition of a state space; that is, whether the condition in Eq.~(\ref{eq:nestingbound}) holds for a \emph{particular} choice of $S, S_1, \cdots, S_k$ and $\Zub_w$. This corresponds to checking whether the nesting property holds at an individual node of a partition tree. If it doesn't, we can refine the partition and repeat the empirical check. 
We are guaranteed to succeed if we repeat until the partition contains only single elements, but empirically find that the check succeeds after a single call to $\refine$ for the upper bound we use.

The use of this adaptive partitioning strategy provides two notable advantages: (a) it frees us to choose \emph{any} upper bounding method, rather than one that can be proven to nest according to a fixed partition tree; and (b) we can customize---and indeed optimize---our partitioning strategy on a per weight function basis. Together, this leads to significant efficiency gains relative to Algorithm~\ref{alg:nesting_sample}.

Algorithm~\ref{alg:nesting_sample_partition} describes our proposed method, \method. It formalizes the adaptive, iterative partitioning strategy and also specifies how the partition tree can be created on-the-fly during sampling without instantiating unnecessary pieces. 
In contrast to Algorithm~\ref{alg:nesting_sample}, \method does not take a fixed partition tree $\mathcal{T}$ as input. Further, it operates with \emph{any} (not necessarily nesting) upper bounding method for (subsets of) the state space of interest. 

Figure~\ref{fig:adaptive_partitioning} illustrates the difference between our adaptive partitioning strategy and a fixed partitioning strategy.  We represent the entire state space as a 2 dimensional square.  The left square in Figure~\ref{fig:adaptive_partitioning} illustrates a fixed partition strategy, as used by \cite{law2009approximately}.  Regardless of the specific weight function defined over the square, the square is always partitioned with alternating horizontal and vertical splits.  To use this fixed partitioning, an upper bound must be proven to nest for any weight function.  In contrast, our adaptive partitioning strategy is illustrated by the right square in Figure~\ref{fig:adaptive_partitioning}, where we choose horizontal or vertical splits based on the particular weight function.  Note that slack is not shown and that the figure illustrates the complete partition trees.

\begin{figure}[hb] 
\centering
\textbf{Fixed vs. Adaptive Partitioning}\par\medskip
\floatbox[{\capbeside\thisfloatsetup{capbesideposition={right,top},capbesidewidth=5.6cm}}]{figure}[\FBwidth]
{\caption{Binary partitioning of a square in the order black, blue, orange, green.  Left: each subspace is split in half according to a predefined partitioning strategy, alternating vertical and horizontal splits.  Right: each subspace is split in half, but the method of splitting (vertical or horizontal) is chosen adaptively with no predefined order.  This figure represents tight upper bounds without slack.}\label{fig:adaptive_partitioning}}
{\includegraphics[width=8cm]{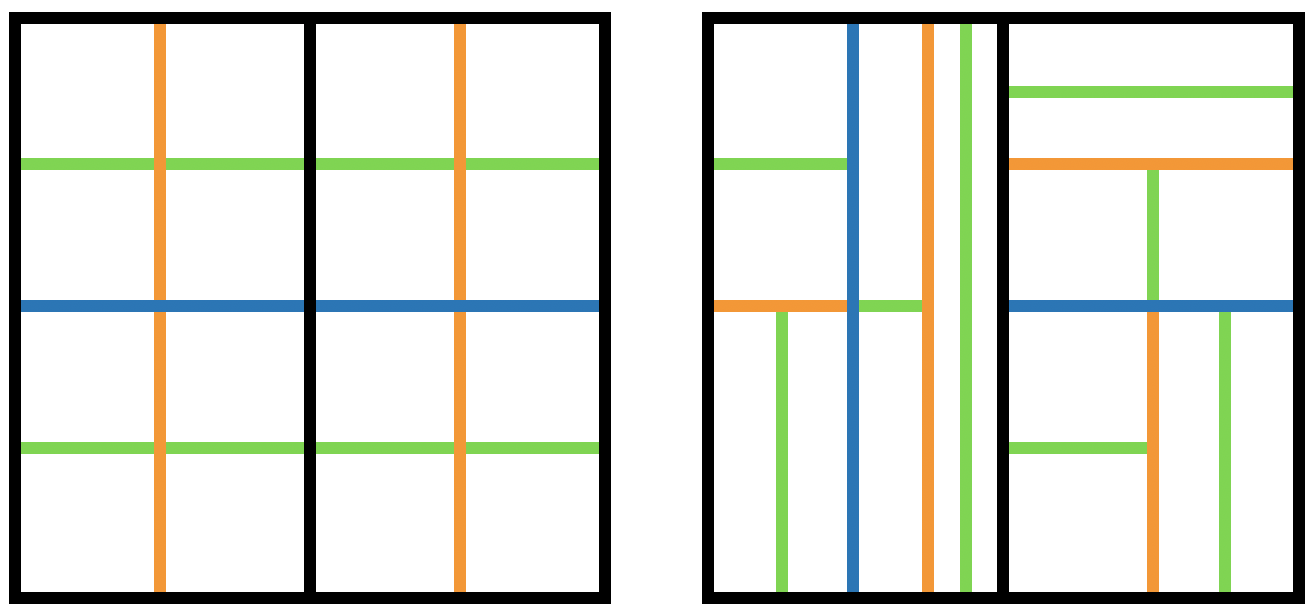}}
\end{figure}

\method uses a function $\refine$, which takes as input a subset $S$ of the state space $\mathcal{S}$, and outputs a collection of $K \geq 1$ different ways of partitioning $S$.  We then use a heuristic to decide which one of these $K$ partitions to keep.  In Figure~\ref{fig:adaptive_partitioning}, $\refine$ takes a rectangle as input and outputs 2 partitionings, the first splitting the rectangle in half horizontally and the second splitting it in half vertically.

\method works as follows. Given a non-negative weight function $w$ for a state space $\mathcal{S}$, we start with the trivial partition $P$ containing only one subset---all of $\mathcal{S}$.  We then call $\refine$ on $\mathcal{S}$, which gives
$K \geq 1$ possible partitions of $\mathcal{S}$.  For each of the $K$ possible partitions, we sum the upper bounds on each subset in the partition, denoting this sum as $\mathit{ub}_i$ for the the $i$-th partition.
At this point, we perform a local optimization step and choose the partition $j$ with the tightest (i.e., smallest) upper bound, $\mathit{ub}_j$. The rest of the $K-1$ options for partitioning $\mathcal{S}$ are discarded at this point. The partition $P$ is `refined' by replacing $\mathcal{S}$ with the disjoint subsets forming the $j$-th partition of $\mathcal{S}$. 

This process is repeated recursively, by calling $\refine$ on another subset $S \in P$, until the sum of upper bounds on all subsets in $P$ is less than the upper bound on $\mathcal{S}$.  We now have a valid nesting partition $P$ of $\mathcal{S}$ and can perform rejection sampling. Similar to Algorithm~\ref{alg:nesting_sample}, we draw a random sample from $P \cup \{\mathit{slack}\}$, where each $S_i \in P$ is chosen with probability $\frac{\Zub_w(S_i)}{\Zub_w(\mathcal{S})}$, and $\mathit{slack}$ is chosen with the remaining probability.  If subset $S_m \in P$ is sampled, we recursively call \method on $S_m$.  If $\mathit{slack}$ is selected, we discard the computation and restart the entire process.  The process stops when $S_m$ is a singleton set $\{a\}$, in which case $a$ is output as the sample. 

\method can be seen as using a greedy approach for optimizing over possible partition trees of $\mathcal{S}$ w.r.t.\ $\Zub_w$. At every node, we partition in a way that minimizes the immediate or ``local'' slack (among the $K$ possible partitioning options). This approach may be sub-optimal due to its greedy nature, but we found it to be efficient and empirically effective.  The efficiency of \method can be improved further by tightening upper bounds whenever slack is encountered, resulting in an \emph{adaptive\footnote{The use of `adaptive' here is to connect this section with the rejection sampling literature, and is unrelated to `adaptive' partitioning discussed earlier.} rejection sampler}~\citep{gilks1992adaptive} (please refer to Section~\ref{sec:adaptive_rejection_sampling} in the Appendix for further details).

\subsection{Estimating the Partition Function}
Armed with a method, \method, for drawing exact samples from a distribution defined by a non-negative weight function $w$ whose partition function $Z_w$ is unknown, we now outline a simple method for using these samples to estimate the partition function $Z_w$.  The acceptance probability of the rejection sampler embedded in \method can be estimated as 
\begin{equation}
\label{eq:p_hat}
    \hat{p} = \frac{\text{accepted samples}}{\text{total samples}} \approx p = \frac{Z_w}{\Zub}
\end{equation}
which yields $\hat{p} \times \Zub$ as an unbiased estimator of $Z_w$. The number of accepted samples out of $T$ total samples is distributed as a Binomial random variable with parameter $p = \frac{Z_w}{\Zub}$.  
The Clopper–Pearson method \cite{clopper1934use} gives tight, high probability bounds on the true acceptance probability, which in turn gives us high probability bounds on $Z_w$.  Please refer to Section~\ref{sec:adaptive_rej_sampling_estimate} in the Appendix for the unbiased estimator of $Z_w$ when performing bound tightening as in an adaptive rejection sampler.

\section{Adaptive Partitioning for the Permanent}

In order to use \method for approximating the permanent of a non-negative matrix $A$, we need to specify two pieces: (a) the $\refine$ method for partitioning any given subset $S$ of the permutations defined by $A$, and (b) a function that upper bounds the permanent of $A$, as well as any subset of the state space (of permutations) generated by $\refine$.

\subsection[Refine for Permutation Partitioning]{$\refine$ for Permutation Partitioning}

We implement the $\refine$ method for partitioning an $n \times n$ matrix into a set of $K = n$ different partitions as follows. One partition is created for each column $i \in \{1, \ldots, n\}$. The $i$-th partition of the $n!$ permutations contains $n$ subsets, corresponding to all permutations containing a matrix element, $\sigma^{-1}(i) = j$ for $j \in \{1, \ldots, n\}$. This is inspired by the fixed partition of \citet[pp. 9-10]{law2009approximately}, modified to choose the column for partitioning adaptively. 



\subsection{Upper Bounding the Permanent}
\label{sec:permanentUBs}

There exists a significant body of work on estimating and bounding the permanent (cf.~an overview by \citet{zhang2016update}), on characterizing the potential tightness of upper bounds \cite{gurvits2006hyperbolic,samorodnitsky2008upper}, 
and on improving upper bounds \cite{hwang1998upper,soules2000extending,soules2003new,soules2005permanental}.  
We use an upper bound from \citet{soules2005permanental}, which is computed as
follows. Define $\gamma(0) = 0$ and $\gamma(k) = (k!)^{1/k}$ for $k \in \mathbb{Z}_{\geq 1}$.
Let $\delta(k) = \gamma(k) - \gamma(k - 1)$.
Given a matrix $A \in \mathbb{R}^{n \times n}$ with entries $A_{ij}$, sort the entries
of each row from largest to smallest to obtain $a^*_{ij}$, where $a^*_{i1}
\geq \dots \geq a^*_{in}$.
This gives the upper bound,
\begin{equation}\label{eq:soulesUB}
   \mathrm{per}(A) \leq \prod_{i=1}^{n} \sum_{j=1}^{n} a^*_{ij} \delta(j).
\end{equation}
If the matrix entries are either 0 or 1, this bound reduces to the Minc-Br\`{e}gman bound \citep{minc1963upper,bregman1973some}.
This upper bound has many desirable properties.  It can be efficiently computed in polynomial time, while tighter bounds (also given by \cite{soules2005permanental}) require solving an optimization problem.  It is significantly tighter than the one used by \citet{law2009approximately}.  This is advantageous because the runtime of \method scales linearly with the bound's tightness (via the acceptance probability of the rejection sampler).

Critically, we empirically find that this bound never requires a second call to $\refine$ in the repeat-until loop of \method.  That is, in practice we always find at least one column that we can partition on to satisfy the nesting condition.  This bounds the number of subsets in a partition to $n$ and avoids a potentially exponential explosion.  This is fortuitous, but also interesting, because this bound (unlike the bound used by \citet{law2009approximately}) does not nest according to any predefined partition tree for all matrices.

\subsection{Dense Matrix Polynomial Runtime Guarantee}


The runtime of \method is bounded for dense matrices as stated in Proposition~\ref{prop:runtime_gaurantee}.  Please refer to Section~\ref{sec:adapart_guarantee_runtime} in the Appendix for further details.

\begin{prop}\label{prop:runtime_gaurantee}
The runtime of \method is $O(n^{1.5 + .5/(2\gamma - 1)})$ for matrices with $\gamma n$ entries in every row and column that all take the maximum value of entries in the matrix, as shown in Algorithm~\ref{alg:adapart_runtime_guarantee}.
\end{prop}

    
    

  

\section{Related Work on Approximating the Permanent}

The fastest exact methods for calculating the permanent have computational complexity that is exponential in the matrix dimension \cite{ryser1963combinatorial,bax1996finite,balasubramanian1980combinatorics,glynn2013permanent}.  This is to be expected, because computing the permanent has been shown to be \#P-complete \cite{valiant1979complexity}.  Work to approximate the permanent has thus followed two parallel tracks, sampling based approaches and variational approaches. 

The sampling line of research has achieved complete (theoretical) success.  \citet{jerrum2004polynomial} proved the existence of a fully polynomial randomized approximation scheme (FPRAS) for approximating the permanent of a general non-negative matrix, which was an outstanding problem at the time \cite{broder1986hard,jerrum1989approximating}.  An FPRAS is the best possible solution that can be reasonably hoped for since computing the permanent is \#P-complete.  Unfortunately, the FPRAS presented by \cite{jerrum2004polynomial} has seen little, if any, practical use.  The algorithm is both difficult to implement and slow with polynomial complexity of $O(n^{10})$, although this complexity was improved to $O(n^7 \log^4 n)$ by \citet{bezakova2006accelerating}.  

In the variational line of research
, the Bethe approximation of the permanent \citep{huang2009approximating,vontobel2014bethe} is guaranteed to be accurate within a factor of ${2}^{n/2}$ \citep{anari2018tight}.  This approach uses belief propagation to minimize the Bethe free energy as a variational objective.  A closely related approximation, using Sinkhorn scaling, is guaranteed to be accurate within a factor of $2^n$ \citep{gurvits2014bounds}.  The difference between these approximations is discussed in \citet{vontobel2014bethe}.  The Sinkhorn based approximation has been shown to converge in polynomial time \citep{linial2000deterministic}, although the authors of \citep{huang2009approximating} could not prove polynomial convergence for the Bethe approximation.  
\citet{aaronson2014generalizing} build on \citep{gurvits2002deterministic} (a precursor to \citep{gurvits2014bounds}) to estimate the permanent in polynomial time within additive error that is exponential in the largest singular value of the matrix.  \jonathan{i'm not really sure if this reference has a variational interpretation, but it does build on work that is by the same author and looks related to the sinkhorn approximation}
While these variational approaches are relatively computationally efficient, their bounds are still exponentially loose.

There is currently a gap between the two lines of research.  The sampling line has found a theoretically ideal FPRAS which is unusable in practice.  The variational line has developed algorithms which have been shown to be both theoretically and empirically efficient, but whose approximations to the permanent are exponentially loose, with only specific cases where the approximations are good \citep{huang2009approximating,chertkov2008belief,chertkov2010inference}.  \citet{huber2006exact} and \citet{law2009approximately} began a new line of sampling research that aims to bridge this gap.  They present a sampling method which is straightforward to implement and 
has a polynomial runtime guarantee for dense matrices.
While there is no runtime guarantee for general matrices, their method is significantly faster than the FRPAS of \cite{jerrum2004polynomial} for dense matrices.  In this paper we present a novel sampling algorithm that builds on the work of \cite{huber2006exact, law2009approximately}.  We show that \method leads to significant empirical speedups, further closing the gap between the sampling and variational lines of research.

\section{Experiments}\label{sec:experiments}

In this section we show the empirical runtime scaling of \method as matrix size increases, test \method on real world matrices, compare \method with the algorithm from \citet{law2009approximately} for sampling from a fixed partition tree, and compare with variational approximations \cite{huang2009approximating,anari2018tight,gurvits2014bounds}.  Please see section \ref{sec:additional_exp} in the Appendix for additional experiments verifying that the permanent empirically falls within our high probability bounds.


\subsection{Runtime Scaling and Comparison with Variational Approximations}

To compare the runtime performance of \method with \citet{law2009approximately} we generated random matrices of varying size.  We generated matrices in two ways, by uniformly sampling every element from $[0,1)$ (referred to as `uniform' in plots) and by sampling $\lfloor \frac{n}{k} \rfloor$ blocks of size $k\text{x}k$ and a single $n\ \mathrm{mod}\ k$ block along the diagonal of an $n\text{x}n$ matrix (with all other elements set to 0, referred to as `block diagonal' in plots).
Runtime scaling is shown in Figure~\ref{fig:matrix_scaling}.
While \method is faster in both cases, we observe the most time reduction for the more challenging, low density block diagonal matrices.  For reference a Cython implementation of Ryser's algorithm for exactly computing the permanent in exponential time \cite{ryser1963combinatorial} requires roughly 1.5 seconds for a $20\text{x}20$ matrix.

\begin{figure}[h]
\centering
\begin{subfigure}{.5\linewidth}
  \centering
  \includegraphics[width=.99\linewidth]{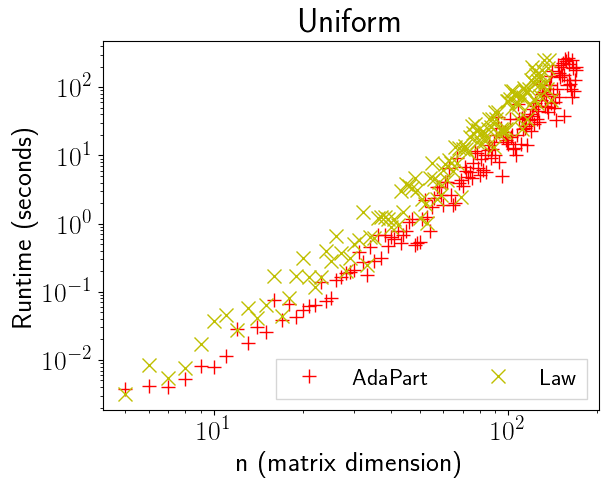}
\end{subfigure}%
\begin{subfigure}{.5\linewidth}
  \centering
  \includegraphics[width=.99\linewidth]{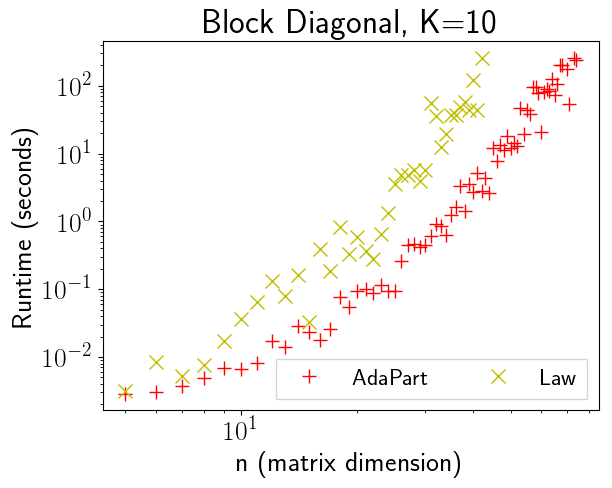}
\end{subfigure}
\caption{Log-log plot of mean runtime over 5 samples against $n$ (matrices are of size $n \times n$).}  
\label{fig:matrix_scaling}
\end{figure}

To demonstrate that computing the permanent of these matrices is challenging for variational approaches, we plot the bounds obtained from the Bethe and Sinkhorn approximations in Figure~\ref{fig:variational_comparison}.  Note that the gap between the variational lower and upper bounds is exponential in the matrix dimension $n$.  Additionally, the upper bound from \citet{soules2005permanental} (that we use in \method) is frequently closer to the exact permanent than all variational bounds.  

\begin{figure}[h]
\centering
\begin{subfigure}{.33\linewidth}
  \centering
  \includegraphics[width=.99\linewidth]{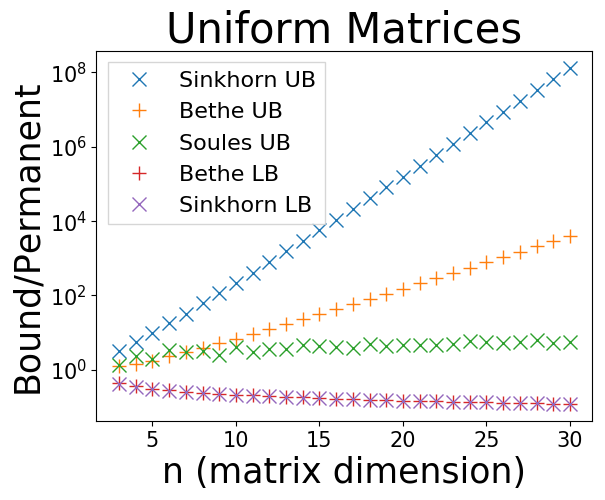}
\end{subfigure}%
\begin{subfigure}{.33\linewidth}
  \centering
  \includegraphics[width=.99\linewidth]{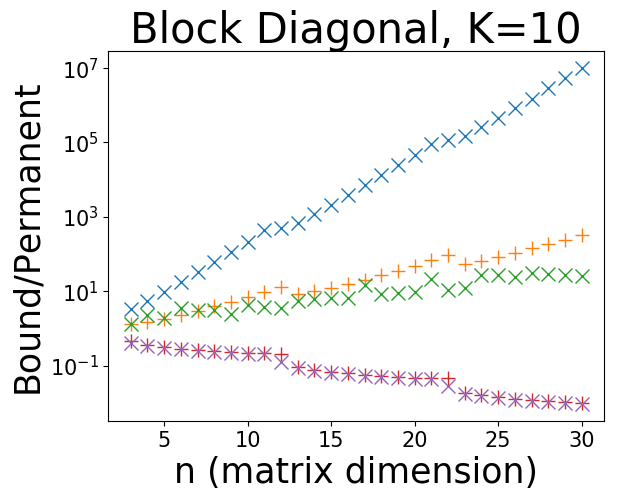}
\end{subfigure}
\begin{subfigure}{.33\linewidth}
  \centering
  \includegraphics[width=.99\linewidth]{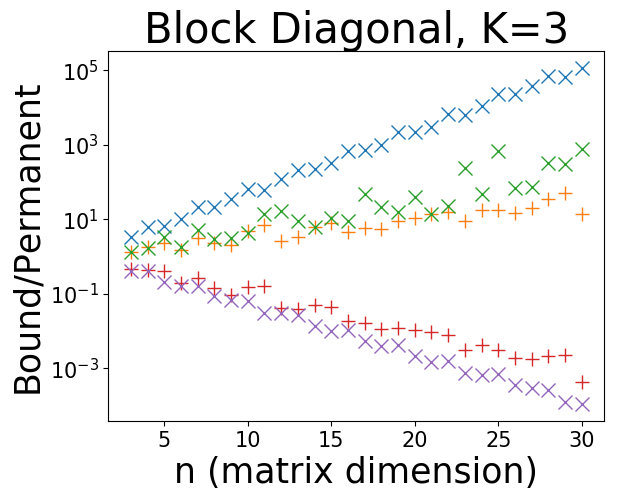}
\end{subfigure}
\caption{Bounds on the permanent given by the Bethe approximation \cite{huang2009approximating,anari2018tight}, the Sinkhorn approximation \cite{gurvits2014bounds}, and the upper bound we use from \citet{soules2005permanental}.}\label{fig:variational_comparison}  
\end{figure}

\subsection{Matrices from Real-World Networks}

In Table~\ref{table:real_matrices} we show the performance of our method on real world problem instances.  In the context of directed graphs, the permanent represents the sum of weights of \emph{cycle covers} (i.e., a set of disjoint directed cycles that together cover all vertices of the graph) and defines a distribution over cycle covers.  Sampling cycle covers is then equivalent to sampling permutations from the distribution defined by the permanent.  We sampled 10 cycle covers from distributions arising from graphs\footnote{Matrices available at http://networkrepository.com.} in the fields of cheminformatics, DNA electrophoresis, and power networks and report mean runtimes in Table~\ref{table:real_matrices}.
Among the matrices that did not time out, \method can sample cycle covers 12 - 25x faster than the baseline from \citet{law2009approximately}.
We used 10 samples from \method to compute bounds on the permanent that are tight within a factor of 5 and hold with probability $.95$, shown in the \method sub-columns of Table~\ref{table:real_matrices} (we show the natural logarithm of all bounds).  Note that we would get comparable bounds using the method from \cite{law2009approximately} as it is also produces exact samples.  For comparison we compute variational bounds using the method of \cite{gurvits2014bounds}, shown in the `Sinkhorn' sub-columns.  Each of these bounds was computed in less than .01 seconds, but they are generally orders of magnitude looser than our sampling bounds.  Note that our sampling bounds can be tightened arbitrarily by using more samples at the cost of additional (parallel) computation, while the Sinkhorn bounds cannot be tightened.  We do not show bounds given by the Bethe approximation because the matlab code from \cite{huang2009approximating} was very slow for matrices of this size and the c++ code does not handle matrices with 0 elements.


\begin{table}[h!] 
\small
\setlength\tabcolsep{4px}
\centering
    \begin{tabular}{@{}lcccccccc@{}}
    \toprule
\multicolumn{3}{c}{Model Information} & \multicolumn{2}{c}{Sampling Runtime (sec.)} &
\multicolumn{2}{c}{Lower Bounds} &
\multicolumn{2}{c}{Upper Bounds}\\
\cmidrule(r){1-3} \cmidrule(r){4-5} \cmidrule(r){6-7} \cmidrule(r){8-9}
 Network Name & Nodes & Edges   & \method & \citet{law2009approximately} & \method & Sinkhorn & \method & Sinkhorn \\\midrule
     ENZYMES-g192             & 31                  & 132 & \textbf{4.2}      & 52.9  & \textbf{19.3} & 17.0  & \textbf{20.8} & 38.5\\
     ENZYMES-g230             & 32                  & 136 & \textbf{3.3}      & 55.5  & \textbf{19.8} & 17.2  & \textbf{21.3} & 39.4\\
     ENZYMES-g479             & 28                  & 98 & \textbf{1.8}      &  45.1  & \textbf{12.3}& 10.9  & \textbf{13.8}& 30.3\\
     cage5             & 37                  & 196  & \textbf{6.1}      &    TIMEOUT  & \textbf{-20.2}& -29.2 & \textbf{-18.7}& -3.6\\
     bcspwr01             & 39                  & 46  & \textbf{4.2}      &  74.8     & \textbf{18.7}& 13.2 &  \textbf{20.1}& 40.3\\
 \bottomrule
    \end{tabular}
\caption{Runtime comparison of our algorithm (\method) with the fixed partitioning algorithm from \citet{law2009approximately} and bound tightness comparison of \method with the Sinkhorn based variational bounds from \cite{gurvits2014bounds} (logarithm of bounds shown).  Best values are in \textbf{bold}.}
\label{table:real_matrices}
\end{table}

\subsection{Multi-Target Tracking}

The connection between measurement association in multi-target tracking and the matrix permanent arises frequently in tracking literature \cite{uhlmann2004matrix,morelande2009joint,oh2009markov,hamid2015joint}.  It is used to calculate the marginal probability that a measurement was produced by a specific target, summing over all other joint measurement-target associations in the association matrix.  
We implemented a Rao-Blackwellized particle filter that uses \method to sample from the optimal proposal distribution and compute approximate importance weights (see Section~\ref{sec:mtt_overview} in the Appendix).  

We evaluated the performance of our particle filter using synthetic multi-target tracking data.  Independent target motion was simulated for 10 targets with linear Gaussian dynamics.  Each target was subjected to a unique spring force.  As baselines, we evaluated against a Rao-Blackwellized particle filter using a sequential proposal distribution \cite{sarkka2004rao} and against the standard multiple hypothesis tracking framework (MHT) \cite{reid1979algorithm,chong2018forty,kim2015multiple}.  We ran each method with varying numbers of particles (or tracked hypothesis in the case of MHT) and plot the maximum log-likelihood of measurement associations among sampled particles in Figure~\ref{fig:tracking_performance}.  The mean squared error over all inferred target locations (for the sampled particle with maximum log-likelihood) is also shown in Figure~\ref{fig:tracking_performance}.  We see that by sampling from the optimal proposal distribution (blue x's in Figure~\ref{fig:tracking_performance}) we can find associations with larger log-likelihood and lower mean squared error than baseline methods while using an order of magnitude fewer samples (or hypotheses in the case of MHT).

\begin{figure}
\centering
\begin{subfigure}{.5\linewidth}
  \centering
  \includegraphics[width=.99\linewidth]{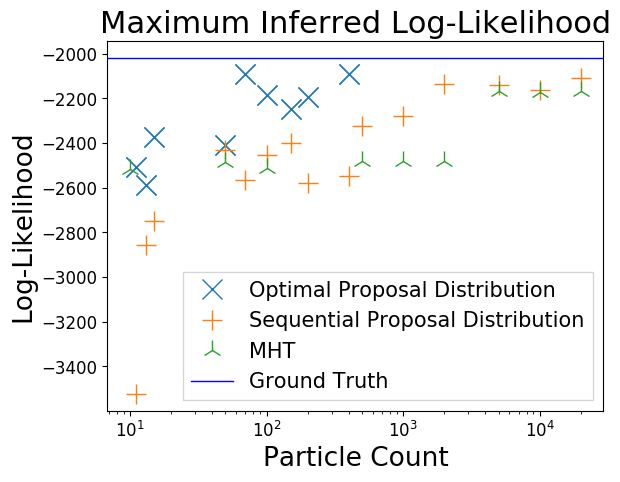}
\end{subfigure}%
\begin{subfigure}{.5\linewidth}
  \centering
  \includegraphics[width=.99\linewidth]{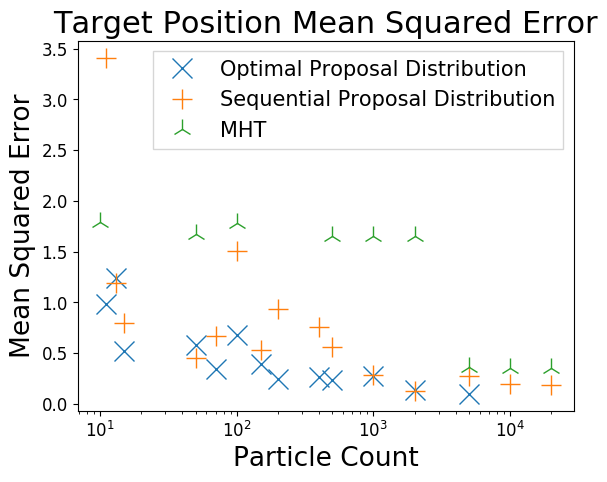}
\end{subfigure}
\caption{Multi-target tracking performance comparison.  Left: maximum log-likelihoods among sampled particles (or top-$k$ hypotheses for the MHT baseline).  Right: mean squared error over all time steps and target locations.
}
\label{fig:tracking_performance}
\end{figure}

\section{Conclusion and Future Work}

Computing the permanent of a matrix is a fundamental problem in computer science. It has many applications, but exact computation is intractable in the worst case. Although a theoretically sound randomized algorithm exists for \emph{approximating} the permanent in polynomial time, it is impractical. We proposed a general approach, \method, for drawing exact samples from an unnormalized distribution.  We used \method to construct high probability bounds on the permanent in provably polynomial time for dense matrices.  We showed that \method is significantly faster than prior work on both dense and sparse matrices which are challenging for variational approaches.  
Finally, we applied \method to the multi-target tracking problem and showed that we can improve tracking performance while using an order of magnitude fewer samples.

In future work, \method may be used to estimate general partition functions if a general upper bound \cite{wainwright2003tree,liu2011bounding,lou2017anytime} is found to nest with few calls to $\refine$.  The matrix permanent specific implementation of \method may benefit from tighter upper bounds on the permanent.  Particularly, a computationally efficient implementation of the Bethe upper bound \citep{huang2009approximating,anari2018tight} would yield improvements on sparse matrices (see Figure~\ref{fig:variational_comparison}), which could be useful for multi-target tracking where the association matrix is frequently sparse. 
The `sharpened' version of the bound we use (Equation~\ref{eq:soulesUB}), also described in \cite{soules2005permanental}, would offer performance improvements if the `sharpening' optimization problem can be solved efficiently.
\subsection*{Acknowledgements}
Research supported by NSF (\#1651565, \#1522054, \#1733686), ONR  (N00014-19-1-2145), AFOSR (FA9550-
19-1-0024), and FLI. 

\nocite{barvinok2016computing}

\bibliography{bibliography}

\begin{thebibliography}{52}
\providecommand{\natexlab}[1]{#1}
\providecommand{\url}[1]{\texttt{#1}}
\expandafter\ifx\csname urlstyle\endcsname\relax
  \providecommand{\doi}[1]{doi: #1}\else
  \providecommand{\doi}{doi: \begingroup \urlstyle{rm}\Url}\fi

\bibitem[Aaronson and Hance(2014)]{aaronson2014generalizing}
Scott Aaronson and Travis Hance.
\newblock Generalizing and derandomizing gurvits's approximation algorithm for
  the permanent.
\newblock \emph{Quantum Information \& Computation}, 14\penalty0
  (7\&8):\penalty0 541--559, 2014.

\bibitem[Anari and Rezaei(2018)]{anari2018tight}
Nima Anari and Alireza Rezaei.
\newblock A tight analysis of bethe approximation for permanent.
\newblock \emph{arXiv preprint arXiv:1811.02933}, 2018.

\bibitem[Atanasov et~al.(2014)Atanasov, Zhu, Daniilidis, and
  Pappas]{atanasov2014semantic}
Nikolay Atanasov, Menglong Zhu, Kostas Daniilidis, and George~J Pappas.
\newblock Semantic localization via the matrix permanent.
\newblock In \emph{Robotics: Science and Systems}, volume~2, 2014.

\bibitem[Balasubramanian(1980)]{balasubramanian1980combinatorics}
K~Balasubramanian.
\newblock \emph{Combinatorics and diagonals of matrices}.
\newblock PhD thesis, Indian Statistical Institute, Calcutta, 1980.

\bibitem[Barvinok(2016)]{barvinok2016computing}
Alexander Barvinok.
\newblock Computing the permanent of (some) complex matrices.
\newblock \emph{Foundations of Computational Mathematics}, 16\penalty0
  (2):\penalty0 329--342, 2016.

\bibitem[Bax and Franklin(1996)]{bax1996finite}
E~Bax and J~Franklin.
\newblock A finite-difference sieve to compute the permanent.
\newblock \emph{CalTech-CS-TR-96-04}, 1996.

\bibitem[Beichl and Sullivan(1999)]{beichl1999approximating}
Isabel Beichl and Francis Sullivan.
\newblock Approximating the permanent via importance sampling with application
  to the dimer covering problem.
\newblock \emph{Journal of computational Physics}, 149\penalty0 (1):\penalty0
  128--147, 1999.

\bibitem[Bez{\'a}kov{\'a} et~al.(2006)Bez{\'a}kov{\'a},
  {\v{S}}tefankovi{\v{c}}, Vazirani, and Vigoda]{bezakova2006accelerating}
Ivona Bez{\'a}kov{\'a}, Daniel {\v{S}}tefankovi{\v{c}}, Vijay~V Vazirani, and
  Eric Vigoda.
\newblock Accelerating simulated annealing for the permanent and combinatorial
  counting problems.
\newblock In \emph{Proceedings of the seventeenth annual ACM-SIAM symposium on
  Discrete algorithm}, pages 900--907, 2006.

\bibitem[Blom and Bloem(2000)]{blom2000probabilistic}
Henk~AP Blom and Edwin~A Bloem.
\newblock Probabilistic data association avoiding track coalescence.
\newblock \emph{IEEE Transactions on Automatic Control}, 45\penalty0
  (2):\penalty0 247--259, 2000.

\bibitem[Bregman(1973)]{bregman1973some}
Lev~M Bregman.
\newblock Some properties of nonnegative matrices and their permanents.
\newblock \emph{Soviet Math. Dokl}, 14\penalty0 (4):\penalty0 945--949, 1973.

\bibitem[Broder(1986)]{broder1986hard}
Andrei~Z Broder.
\newblock How hard is it to marry at random?(on the approximation of the
  permanent).
\newblock In \emph{Proceedings of the eighteenth annual ACM symposium on Theory
  of computing}, pages 50--58. ACM, 1986.

\bibitem[Chernick(2008)]{chernick2008bootstrap}
Michael~R Chernick.
\newblock Bootstrap methods: A guide for practitioners and researchers.
  hoboken, 2008.

\bibitem[Chertkov et~al.(2008)Chertkov, Kroc, and
  Vergassola]{chertkov2008belief}
Michael Chertkov, Lukas Kroc, and Massimo Vergassola.
\newblock Belief propagation and beyond for particle tracking.
\newblock \emph{arXiv preprint arXiv:0806.1199}, 2008.

\bibitem[Chertkov et~al.(2010)Chertkov, Kroc, Krzakala, Vergassola, and
  Zdeborov{\'a}]{chertkov2010inference}
Michael Chertkov, Lukas Kroc, F~Krzakala, M~Vergassola, and L~Zdeborov{\'a}.
\newblock Inference in particle tracking experiments by passing messages
  between images.
\newblock \emph{Proceedings of the National Academy of Sciences}, 107\penalty0
  (17):\penalty0 7663--7668, 2010.

\bibitem[Chong et~al.(2018)Chong, Mori, and Reid]{chong2018forty}
Chee-Yee Chong, Shozo Mori, and Donald~B Reid.
\newblock Forty years of multiple hypothesis tracking-a review of key
  developments.
\newblock In \emph{2018 21st International Conference on Information Fusion
  (FUSION)}, pages 452--459. IEEE, 2018.

\bibitem[Clopper and Pearson(1934)]{clopper1934use}
Charles~J Clopper and Egon~S Pearson.
\newblock The use of confidence or fiducial limits illustrated in the case of
  the binomial.
\newblock \emph{Biometrika}, 26\penalty0 (4):\penalty0 404--413, 1934.

\bibitem[Doucet et~al.(2000)Doucet, Godsill, and Andrieu]{doucet2000sequential}
Arnaud Doucet, Simon Godsill, and Christophe Andrieu.
\newblock On sequential monte carlo sampling methods for bayesian filtering.
\newblock \emph{Statistics and computing}, 10\penalty0 (3):\penalty0 197--208,
  2000.

\bibitem[Fortmann et~al.(1983)Fortmann, Bar-Shalom, and
  Scheffe]{fortmann1983sonar}
Thomas Fortmann, Yaakov Bar-Shalom, and Molly Scheffe.
\newblock Sonar tracking of multiple targets using joint probabilistic data
  association.
\newblock \emph{IEEE journal of Oceanic Engineering}, 8\penalty0 (3):\penalty0
  173--184, 1983.

\bibitem[Gilks and Wild(1992)]{gilks1992adaptive}
Walter~R Gilks and Pascal Wild.
\newblock Adaptive rejection sampling for gibbs sampling.
\newblock \emph{Applied Statistics}, pages 337--348, 1992.

\bibitem[Glynn(2013)]{glynn2013permanent}
David~G Glynn.
\newblock Permanent formulae from the veronesean.
\newblock \emph{Designs, codes and cryptography}, 68\penalty0 (1-3):\penalty0
  39--47, 2013.

\bibitem[Gurvits(2006)]{gurvits2006hyperbolic}
Leonid Gurvits.
\newblock Hyperbolic polynomials approach to van der waerden/schrijver-valiant
  like conjectures: sharper bounds, simpler proofs and algorithmic
  applications.
\newblock In \emph{Proceedings of the thirty-eighth annual ACM symposium on
  Theory of computing}, pages 417--426. ACM, 2006.

\bibitem[Gurvits and Samorodnitsky(2002)]{gurvits2002deterministic}
Leonid Gurvits and Alex Samorodnitsky.
\newblock A deterministic algorithm for approximating the mixed discriminant
  and mixed volume, and a combinatorial corollary.
\newblock \emph{Discrete \& Computational Geometry}, 27\penalty0 (4):\penalty0
  531--550, 2002.

\bibitem[Gurvits and Samorodnitsky(2014)]{gurvits2014bounds}
Leonid Gurvits and Alex Samorodnitsky.
\newblock Bounds on the permanent and some applications.
\newblock In \emph{2014 IEEE 55th Annual Symposium on Foundations of Computer
  Science}, pages 90--99. IEEE, 2014.

\bibitem[Huang and Jebara(2009)]{huang2009approximating}
Bert Huang and Tony Jebara.
\newblock Approximating the permanent with belief propagation.
\newblock \emph{arXiv preprint arXiv:0908.1769}, 2009.

\bibitem[Huber(2006)]{huber2006exact}
Mark Huber.
\newblock Exact sampling from perfect matchings of dense regular bipartite
  graphs.
\newblock \emph{Algorithmica}, 44\penalty0 (3):\penalty0 183--193, 2006.

\bibitem[Hwang et~al.(1998)Hwang, Kr{\"a}uter, and Michael]{hwang1998upper}
Suk-Geun Hwang, Arnold~R Kr{\"a}uter, and TS~Michael.
\newblock An upper bound for the permanent of a nonnegative matrix.
\newblock \emph{Linear algebra and its applications}, 281\penalty0
  (1-3):\penalty0 259--263, 1998.

\bibitem[Jerrum and Sinclair(1989)]{jerrum1989approximating}
Mark Jerrum and Alistair Sinclair.
\newblock Approximating the permanent.
\newblock \emph{SIAM journal on computing}, 18\penalty0 (6):\penalty0
  1149--1178, 1989.

\bibitem[Jerrum and Sinclair(1996)]{jerrum1996markov}
Mark Jerrum and Alistair Sinclair.
\newblock The markov chain monte carlo method: an approach to approximate
  counting and integration.
\newblock \emph{Approximation algorithms for NP-hard problems}, pages 482--520,
  1996.

\bibitem[Jerrum et~al.(2004)Jerrum, Sinclair, and Vigoda]{jerrum2004polynomial}
Mark Jerrum, Alistair Sinclair, and Eric Vigoda.
\newblock A polynomial-time approximation algorithm for the permanent of a
  matrix with nonnegative entries.
\newblock \emph{Journal of the ACM (JACM)}, 51\penalty0 (4):\penalty0 671--697,
  2004.

\bibitem[Jerrum et~al.(1986)Jerrum, Valiant, and Vazirani]{jerrum1986random}
Mark~R Jerrum, Leslie~G Valiant, and Vijay~V Vazirani.
\newblock Random generation of combinatorial structures from a uniform
  distribution.
\newblock \emph{Theoretical Computer Science}, 43:\penalty0 169--188, 1986.

\bibitem[Kim et~al.(2015)Kim, Li, Ciptadi, and Rehg]{kim2015multiple}
Chanho Kim, Fuxin Li, Arridhana Ciptadi, and James~M Rehg.
\newblock Multiple hypothesis tracking revisited.
\newblock In \emph{Proceedings of the IEEE International Conference on Computer
  Vision}, pages 4696--4704, 2015.

\bibitem[Law(2009)]{law2009approximately}
Wai~Jing Law.
\newblock \emph{Approximately counting perfect and general matchings in
  bipartite and general graphs}.
\newblock PhD thesis, Dept.~of Mathematics, Duke University, 2009.

\bibitem[Linial et~al.(2000)Linial, Samorodnitsky, and
  Wigderson]{linial2000deterministic}
Nathan Linial, Alex Samorodnitsky, and Avi Wigderson.
\newblock A deterministic strongly polynomial algorithm for matrix scaling and
  approximate permanents.
\newblock \emph{Combinatorica}, 20\penalty0 (4):\penalty0 545--568, 2000.

\bibitem[Liu and Ihler(2011)]{liu2011bounding}
Qiang Liu and Alexander~T Ihler.
\newblock Bounding the partition function using holder's inequality.
\newblock In \emph{ICML}, 2011.

\bibitem[Lou et~al.(2017)Lou, Dechter, and Ihler]{lou2017anytime}
Qi~Lou, Rina Dechter, and Alexander Ihler.
\newblock Anytime anyspace and/or search for bounding the partition function.
\newblock In \emph{AAAI}, 2017.

\bibitem[Minc(1963)]{minc1963upper}
Henryk Minc.
\newblock Upper bounds for permanents of (0, 1)-matrices.
\newblock \emph{Bulletin of the American Mathematical Society}, 69\penalty0
  (6):\penalty0 789--791, 1963.

\bibitem[Morelande(2009)]{morelande2009joint}
Mark~R Morelande.
\newblock Joint data association using importance sampling.
\newblock In \emph{Information Fusion, 2009. FUSION'09. 12th International
  Conference on}, pages 292--299. IEEE, 2009.

\bibitem[Oh et~al.(2009)Oh, Russell, and Sastry]{oh2009markov}
Songhwai Oh, Stuart Russell, and Shankar Sastry.
\newblock Markov chain monte carlo data association for multi-target tracking.
\newblock \emph{IEEE Transactions on Automatic Control}, 54\penalty0
  (3):\penalty0 481--497, 2009.

\bibitem[Reid et~al.(1979)]{reid1979algorithm}
Donald Reid et~al.
\newblock An algorithm for tracking multiple targets.
\newblock \emph{IEEE transactions on Automatic Control}, 24\penalty0
  (6):\penalty0 843--854, 1979.

\bibitem[Rezatofighi et~al.(2015)Rezatofighi, Milan, Zhang, Shi, Dick, and
  Reid]{hamid2015joint}
{S. Hamid} Rezatofighi, Anton Milan, Zhen Zhang, Qinfeng Shi, Anthony Dick, and
  Ian Reid.
\newblock Joint probabilistic data association revisited.
\newblock In \emph{Proceedings of the IEEE international conference on computer
  vision}, pages 3047--3055, 2015.

\bibitem[Ryser(1963)]{ryser1963combinatorial}
Herbert~John Ryser.
\newblock \emph{Combinatorial mathematics}.
\newblock Mathematical Association of America; distributed by Wiley, New York,
  1963.

\bibitem[Samorodnitsky(2008)]{samorodnitsky2008upper}
Alex Samorodnitsky.
\newblock An upper bound for permanents of nonnegative matrices.
\newblock \emph{Journal of Combinatorial Theory, Series A}, 115\penalty0
  (2):\penalty0 279--292, 2008.

\bibitem[S{\"a}rkk{\"a} et~al.(2004)S{\"a}rkk{\"a}, Vehtari, and
  Lampinen]{sarkka2004rao}
Simo S{\"a}rkk{\"a}, Aki Vehtari, and Jouko Lampinen.
\newblock Rao-blackwellized monte carlo data association for multiple target
  tracking.
\newblock In \emph{Proceedings of the seventh international conference on
  information fusion}, volume~1, pages 583--590. I, 2004.

\bibitem[Soules(2000)]{soules2000extending}
George~W Soules.
\newblock Extending the minc-bregman upper bound for the permanent.
\newblock \emph{Linear and Multilinear Algebra}, 47\penalty0 (1):\penalty0
  77--91, 2000.

\bibitem[Soules(2003)]{soules2003new}
George~W Soules.
\newblock New permanental upper bounds for nonnegative matrices.
\newblock \emph{Linear and Multilinear Algebra}, 51\penalty0 (4):\penalty0
  319--337, 2003.

\bibitem[Soules(2005)]{soules2005permanental}
George~W Soules.
\newblock Permanental bounds for nonnegative matrices via decomposition.
\newblock \emph{Linear algebra and its applications}, 394:\penalty0 73--89,
  2005.

\bibitem[Uhlmann(2004)]{uhlmann2004matrix}
Jeffrey~K Uhlmann.
\newblock Matrix permanent inequalities for approximating joint assignment
  matrices in tracking systems.
\newblock \emph{Journal of the Franklin Institute}, 341\penalty0 (7):\penalty0
  569--593, 2004.

\bibitem[Valiant(1979)]{valiant1979complexity}
Leslie~G Valiant.
\newblock The complexity of computing the permanent.
\newblock \emph{Theoretical computer science}, 8\penalty0 (2):\penalty0
  189--201, 1979.

\bibitem[Vontobel(2014)]{vontobel2014bethe}
Pascal~O Vontobel.
\newblock The bethe and sinkhorn approximations of the pattern maximum
  likelihood estimate and their connections to the valiant-valiant estimate.
\newblock In \emph{2014 Information Theory and Applications Workshop (ITA)},
  pages 1--10. IEEE, 2014.

\bibitem[Wainwright et~al.(2003)Wainwright, Jaakkola, and
  Willsky]{wainwright2003tree}
Martin~J Wainwright, Tommi~S Jaakkola, and Alan~S Willsky.
\newblock Tree-reweighted belief propagation algorithms and approximate ml
  estimation by pseudo-moment matching.
\newblock In \emph{AISTATS}, 2003.

\bibitem[Wainwright et~al.(2008)Wainwright, Jordan,
  et~al.]{wainwright2008graphical}
Martin~J Wainwright, Michael~I Jordan, et~al.
\newblock Graphical models, exponential families, and variational inference.
\newblock \emph{Foundations and Trends{\textregistered} in Machine Learning},
  1\penalty0 (1--2):\penalty0 1--305, 2008.

\bibitem[Zhang(2016)]{zhang2016update}
Fuzhen Zhang.
\newblock An update on a few permanent conjectures.
\newblock \emph{Special Matrices}, 4\penalty0 (1), 2016.

\end{thebibliography}
\bibliographystyle{plainnat}

\clearpage
\appendix
\section{Appendix}
\subsection{Proof of Correctness for Sampling with a Fixed Partition}
Algorithm \ref{alg:nesting_sample} specifies a method for sampling from a weight function given a fixed partition tree and a bound that provably nests.  Its proof of correctness is given in Proposition\ref{prop:rejection_sampling_correctness}.  Note that a simple property that follows from recursively applying the definition of a nesting bound is that
$\sum_{i \in \mathcal{S}} w(i) \leq \Zub_w (\mathcal{S})$.
More generally, given any node $v$ in $\mathcal{T}$ associated with the subset $S_v \subseteq \mathcal{S}$, we have
$\sum_{i \in S_v} w(i) \leq \Zub_w (S_v)$.

\begin{prop}[\citet{huber2006exact,law2009approximately}]
\label{prop:rejection_sampling_correctness}
Algorithm~\ref{alg:nesting_sample} samples an element $i \in \mathcal{S}$ from the normalized weight function $i \sim \frac{w(i)}{\sum_{j \in \mathcal{S}} w(j)}$.
\end{prop}
\begin{proof} The probability of sampling leaf node $v_i$ at depth $d$ in the partition tree, with ancestors $v^a_{d-1}$, \dots, $v^a_0$ (where $v^a_{d-1}$ is the parent node of $v_i$ and $v^a_0$ is the root node) and associated ancestor subsets $S^a_{d-1}$, \dots, $S^a_0$ is
\begin{align*}
& \frac{1}{p_{accept}} \times \frac{Z^{UB}_w(S^a_1)}{Z^{UB}_w(S^a_0)} \times \frac{Z^{UB}_w(S^a_2)}{Z^{UB}_w(S^a_1)} \times \dots \times \frac{Z^{UB}_w(S^a_d)}{Z^{UB}_w(S^a_{d-1})} \\
= & \frac{1}{p_{accept}} \times \frac{Z^{UB}_w(S^a_d)}{Z^{UB}_w(S^a_0)} = \frac{Z^{UB}_w(\mathcal{S})}{Z_w} \times \frac{w(i)}{Z^{UB}_w(\mathcal{S})}
= \frac{w(i)}{\sum_{i \in \mathcal{S}} w(i)} \\
\end{align*}
\end{proof}

\begin{algorithm}[h]
  \caption{Sample from a Normalized Weight Function}
  \label{alg:nesting_sample}
\begin{flushleft}
  \textbf{Inputs:} 
\end{flushleft}
  
\begin{enumerate}
    \item Non-empty state space $\mathcal{S} = \{1, \dots, N\}$
    \item Partition tree $\mathcal{T}$ of $\mathcal{S}$
    \item Unnormalized weight function $w:\mathcal{S} \rightarrow \mathbb{R}_{\geq 0}$
    \item Nesting upper bound $\Zub_w(S)$ for $w$ with respect to  $\mathcal{T}$ 
\end{enumerate}

\begin{flushleft}
  \textbf{Output:} A sample $i \in \mathcal{S}$ distributed as $i \sim \frac{w(i)}{\sum_{j \in \mathcal{S}} w(j)}$.
\end{flushleft}
  
\begin{flushleft}
  \textbf{Algorithm:}
\end{flushleft}
  
\begin{enumerate}
    \item Set $v$ to the root node of $\mathcal{T}$ and $S = \mathcal{S}$.
    \item Sample a child of $v$ (denoted $v_1, \dots, v_k$ with associated subsets $S_1, \dots, S_k$ of $\mathcal{S}$) or slack with probabilities:
    \begin{align*}
        p(v_l) & = \frac{\Zub_w(S_l)}{\Zub_w(S)} \hspace{3em}
        p(\text{slack})  = 1 - \frac{\sum_{l = 1}^k \Zub_w(S_l)}{\Zub_w(S)}      
    \end{align*}
    \item If a child was sampled with an associated subset containing a single element then return this element.
    \item If a child, $v_l$, was sampled with an associated subset containing more than one element then set $v = v_l$, $S = S_l$, and go to step 2.
    \item If the slack element was sampled then go to step 1.
    
\end{enumerate}
\end{algorithm}

\subsection{Adaptive Rejection Sampling}\label{sec:adaptive_rejection_sampling}

We can improve the efficiency of \method by tightening the upper bounds $\Zub_w$ whenever we encounter slack. This is done by subtracting the computed slack from the associated upper bounds, which still preserves nesting properties. 
The resulting algorithm is an \emph{adaptive rejection sampler}~\citep{gilks1992adaptive}, where the ``envelope'' proposal is tightened every time a point is rejected.\footnote{The use of `adaptive' here is to connect this section with the rejection sampling literature, and is unrelated to `adaptive' partitioning discussed earlier.}

Formally, for any partition $P$ of $S$, we define a new, tighter upper bound as follows:
\begin{equation}
  \Zubunder_w(S) = \min \left\{ \sum_{S_i \in P} \Zub_w(S_i),\, \Zub_w(S) \right\}.
\end{equation}
This is still a valid upper bound
on $\Zub_w(S)$ because of the additive nature of $Z_w$,
and is, by definition, also nesting w.r.t.\ the partition $P$. 
If we encounter any slack, there must exists some $S$ for which $\Zubunder_w(S) <  \Zub_w(S)$, hence we can \emph{strictly} tighten our bound for subsequent steps of the algorithm (thereby making \method more efficient) by using $\Zubunder_w(S)$ instead of $\Zub_w(S)$.
For matrices sampled uniformly from $[0,1)$, we empirically find that bound tightening is most effective for small matrices.  After 1000 samples we improve our bound on the permanent to roughly 64\%, 77\%, and 89\% of the original bound for matrices of size 10, 15, 25 respectively.  However, bound tightening is more effective for other types of matrices.  For the matrices from real world networks in Section~\ref{sec:experiments}, after drawing 10 samples we improve our bound on the permanent to roughly 25\%, 22\%, 8\%, 7\%, and 19\% for models ENZYMES-g192, ENZYMES-g230, ENZYMES-g479, cage5, and bcspwr01 respectively.  


\subsection{Estimating the Partition Function with Adaptive Rejection Sampling}\label{sec:adaptive_rej_sampling_estimate}


The number of accepted samples, $a$, is a random variable with expectation $E[a] = \sum_{i=1}^T \frac{Z}{Z^{UB}_i}$, where $Z^{UB}_i$ is the upper bound on the entire state space $\mathcal{S}$ when the $i$-th sample is drawn.  This gives the unbiased estimator $\hat{Z} = a/\left(\sum_{i=1}^T \frac{1}{Z^{UB}_i}\right)$ for the partition function.  We use bootstrap techniques \cite{chernick2008bootstrap} to perform Monte Carlo simulations that yield high probability bounds on the partition function.  We used 100,000 samples to compute our bounds in Table~\ref{table:real_matrices}, which took 2-6 seconds using our python script, but note that this computation is extremely parallelizable.  We computed each bootstrap upper and lower bound 10 times and always sampled the same values, except for one case where the log upper bound changed by .06, showing that 100,000 samples is sufficient.


\subsection{Runtime Guarantee of \method}\label{sec:adapart_guarantee_runtime}
\citet{law2009approximately} prove that the runtime of Algorithm~\ref{alg:nesting_sample} is  $O(n^{1.5 + .5/(2\gamma - 1)})$ per sample when using their upper bound on the permanent \citep[p. 33]{law2009approximately}, where $\gamma$ controls density.  \method has the same guarantee with a minor modification to the presentation in Algorithm~\ref{alg:nesting_sample_partition}.  The repeat loop is removed and if the terminating condition $\mathit{ub} \leq \Zub_w(\mathcal{S})$ is not met after a single call to $\refine$, Algorithm~\ref{alg:nesting_sample} is called with the upper bound from and fixed partitioning strategy from \citep{law2009approximately} as shown in Algorithm~\ref{alg:adapart_runtime_guarantee}.

\begin{algorithm}[t!]
  \caption{\method: Sample from a Normalized Weight Function using Adaptive Partitioning with Polynomial Runtime Guarantee for Dense Matrices}
  \label{alg:adapart_runtime_guarantee}
  \begin{flushleft}
  \textbf{Inputs:} 
  \end{flushleft}
\begin{enumerate}
    \item Non-empty state space $\mathcal{S}$
    \item Unnormalized weight function $w: \mathcal{S} \rightarrow \mathbb{R}_{\geq 0}$
    \item Family of upper bounds 
    $\Zub_w(S):\mathcal{D} \subseteq 2^{\mathcal{S}} \rightarrow \mathbb{R}_{\geq 0}$ for $w$ that are tight on single element subsets
    \item Refinement function $\refine: \mathcal{P} \rightarrow 2^{\mathcal{P}}$ where $\mathcal{P}$ is the set of all partitions of $\mathcal{S}$
\end{enumerate}
  \begin{flushleft}
  \textbf{Output:} A sample $i \in \mathcal{S}$ distributed as $i \sim \frac{w(i)}{\sum_{i \in \mathcal{S}} w(i)}$.
   \end{flushleft}

 \begin{algorithmic}
 \STATE \algorithmicif\ {$\mathcal{S}=\{a\}$} \algorithmicthen\ {Return $a$}
 
 \STATE $\mathit{ub} \leftarrow \Zub_w(\mathcal{S})$
   \STATE $\{\{S^i_1, \cdots, S^i_{\ell_i}\}\}_{i=1}^K \leftarrow \refine(\mathcal{S})$
   \FORALL {$i \in \{1, \cdots, K\}$}
     \STATE $\mathit{ub}_i \leftarrow \sum_{j=1}^{\ell_i} \Zub_w(S^i_j)$
   \ENDFOR
   \STATE $j \leftarrow \arg \min_i \mathit{ub}_i$
   \STATE $P \leftarrow \{S^j_1, \cdots, S^j_{\ell_j}\}$
   \STATE $\mathit{ub} \leftarrow \mathit{ub} - \Zub_w(S) + \mathit{ub}_j$
\IF {$\mathit{ub} > \Zub_w(\mathcal{S})$}
    \STATE Return the output of Algorithm~\ref{alg:nesting_sample} called on $\mathcal{S}$ and $w$ with the bound and fixed partition of \citep{law2009approximately}
\ELSE
 \STATE Sample a subset $S_i \in P$ with prob.\
 $\frac{\Zub_w(S_i)}{\Zub_w(\mathcal{S})}$, or sample $\mathit{slack}$ with prob.~$1-\frac{\mathit{ub}}{\Zub_w(\mathcal{S})}$
\ENDIF
 \IF {$S_m \in P$ is sampled}
   \STATE Recursively call \method($S_m,w,\Zub_w,\refine$)
 \ELSE
   \STATE Restart, i.e., call \method($\mathcal{S},w,\Zub_w,\refine$)
 \ENDIF
 \end{algorithmic}
\end{algorithm}

\subsection{Additional Experiments} \label{sec:additional_exp}

\begin{figure}[ht]
\centering
\includegraphics[width=8cm]{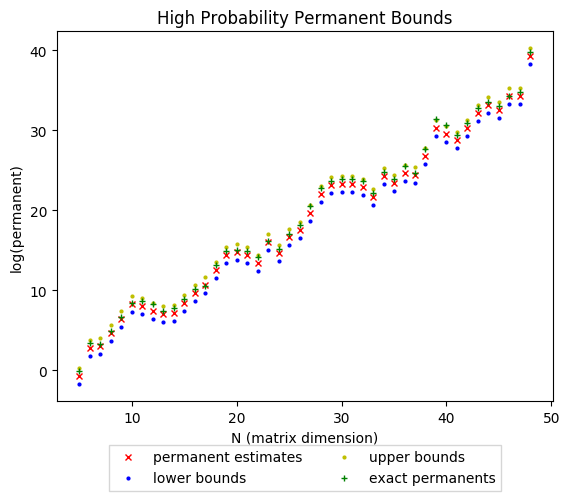}
\caption{Accuracy results on randomly sampled $n \text{x} n$ block diagonal matrices constructed as described earlier, with blocks of size $k=10$.  We plot the exact permanent, our estimate, and our high probability bounds calculated from 10 samples for each matrix. 
}  \label{fig:diag_correctness} 
\end{figure}

While calculating the permanent of a large matrix is generally intractable, it can be done efficiently for certain special types of matrices.  One example is block diagonal matrices, where an $n\text{x}n$ matrix is composed of  $\lfloor \frac{n}{k} \rfloor$ blocks of size $k\text{x}k$ and a single $n\ \mathrm{mod}\ k$ block along the diagonal.  Only elements within these blocks on the diagonal may be non-zero.  The permanent of a block diagonal matrix is simply the product of the permanents of each matrix along the diagonal, which can be calculated efficiently whenever the block size is sufficiently small.  We plot the exact permanent, our estimate, and our high probability bounds for randomly sampled block diagonal matrices of various sizes in Figure~\ref{fig:diag_correctness}.
\subsection{Multi-Target Tracking Overview}\label{sec:mtt_overview}
The multi-target tracking problem is very similar to classical inference problems in hidden Markov models, requiring the estimation of an unobserved state given a time series of noisy measurements.  The non-standard catch is that at each time step the observer is given one noisy measurement per target, but is not told which target produced which measurement.  Brute forcing the problem is intractable because there are $K!$ potential associations when tracking $K$ targets.  The connection between measurement association and the matrix permanent arises frequently in tracking literature \cite{uhlmann2004matrix,morelande2009joint,oh2009markov,hamid2015joint}, and its computational complexity is cited when discussing the difficulty of multi-target tracking.  

As brief background, the computational complexity of multi-target tracking has led to many heuristic approximations, notably including multiple hypothesis tracking (MHT) \cite{reid1979algorithm,chong2018forty,kim2015multiple} and joint probabilistic data association (JPDA) \cite{fortmann1983sonar,hamid2015joint}.  As heuristics, they can succumb to failure modes.  JPDA is known to suffer from target coalescence where neighboring tracks merge \cite{blom2000probabilistic}.

Alternatively, sequential Monte Carlo methods (SMC or particle filters) provide an asymptotically unbiased method for sequentially sampling from arbitrarily complex distributions.  When targets follow linear Gaussian dynamics, a Rao-Blackwellized particle filter may be used to sample the measurement associations allowing sufficient statistics for distributions over individual target states to be computed in closed form (by Kalman filtering, see Algorithm~\ref{alg:SIS} in the Appendix for further details) \cite{sarkka2004rao}.  The proposal distribution is a primary limitation when using Monte Carlo methods.  Ideally it should match the target distribution as closely as possible, but this generally makes it computationally unwieldy.  

In the case of a Rao-Blackwellized particle filter for multi-target tracking, the optimal proposal distribution \citep[p.~199]{doucet2000sequential} that minimizes the variance of each importance weight is a distribution over permutations defined by a matrix permanent (please see Section~\ref{sec:optimal_proposal_distribution} in the Appendix for further details).  We implemented a Rao-Blackwellized particle filter that uses the optimal proposal distribution.  We evaluated it's effectiveness against a Rao-Blackwellized particle filter using a sequential proposal distribution \cite{sarkka2004rao} \jonathan{check me} and against the standard multiple hypothesis tracking framework (MHT) \cite{reid1979algorithm,chong2018forty,kim2015multiple}.  

Our work can be extended to deal with a variable number of targets and clutter measurements using a matrix formulation similar to that in \cite{atanasov2014semantic}. \jonathan{maybe check/elaborate}

\subsection{Optimal Single-Target Bayesian Filtering}
\jonathan{this section has issues with variable collisions that should be resolved, or explain to the reader that x and y refer to single target states in this section only or something.  Add citations HMM, Kalman filter}
In this section we give a brief review of the optimal Bayesian filter for single-target tracking.  Consider a hidden Markov model with unobserved state $\mathbf{x}_t$ and measurement $\mathbf{y}_t$ at time $t$.  The joint distribution over states and measurements factors as 
\begin{equation*}
    \Pr(\mathbf{x}_{1:T}, \mathbf{y}_{1:T}) = \Pr(\mathbf{x}_{1}) \Pr(\mathbf{y}_1 | \mathbf{x}_{1}) \prod_{t=2}^{T} \Pr(\mathbf{x}_{t} | \mathbf{x}_{t-1}) \Pr(\mathbf{y}_t | \mathbf{x}_{t})
\end{equation*}
by the Markov property.  This factorization of the joint distribution facilitates Bayesian filtering, a recursive algorithm that maintains a fully Bayesian distribution over the hidden state $\mathbf{x}_{t}$ as each measurement $\mathbf{y}_{t}$ is sequentially observed.  Given the prior distribution $p(\mathbf{x}_{1})$ over the initial state, the Bayesian filter consists of the update step\footnote{Where we have abused notation and the initial distribution is $\Pr(\mathbf{x}_{1} | \mathbf{y}_{1:0}) = \Pr(\mathbf{x}_{1})$.}
\begin{equation*}
    \Pr(\mathbf{x}_{t} | \mathbf{y}_{1:t}) = \frac{\Pr(\mathbf{y}_{t} | \mathbf{x}_{t}) \Pr(\mathbf{x}_{t} | \mathbf{y}_{1:t-1})}{\int \Pr(\mathbf{y}_{t} | \mathbf{x}_{t}) \Pr(\mathbf{x}_{t} | \mathbf{y}_{1:t-1}) d\mathbf{x}_{t}}
\end{equation*}
and the prediction step
\begin{equation*}
    \Pr(\mathbf{x}_{t} | \mathbf{y}_{1:t-1}) = \int \Pr(\mathbf{x}_{t} | \mathbf{x}_{t-1}) \Pr(\mathbf{x}_{t-1} | \mathbf{y}_{1:t-1}) d\mathbf{x}_{t-1}.
\end{equation*}

In the special case of linear Gaussian models where the state transition and measurement processes are linear but corrupted with Gaussian noise, the above integrals can be computed analytically giving closed form update and predict steps.  The distribution over the hidden states remains Gaussian and is given by the Kalman filter \jonathan{add citation, possibly give formulas in appendix?} with update step
\begin{equation} \label{eq:kf_update}
    \Pr(\mathbf{x}_{t} | \mathbf{y}_{1:t}) = \mathcal{N}(\hat{\mathbf{x}}_{t|t}, \mathbf{P}_{t|t})
\end{equation}
and prediction step
\begin{equation} \label{eq:kf_predict}
    \Pr(\mathbf{x}_{t} | \mathbf{y}_{1:t-1}) = \mathcal{N}(\hat{\mathbf{x}}_{t|t-1}, \mathbf{P}_{t|t-1}).
\end{equation}

\subsection{Optimal Multi-Target Bayesian Filtering}

In this section we give a brief review of the optimal Bayesian filter for multi-target tracking problem with a fixed cardinality (fixed number of targets and measurements over time) \citep[pp.~485-486]{oh2009markov} and its computational intractability. \jonathan{Move me: We then show how to perform sequential Monte Carlo (SMC) on the distribution defined by the optimal Bayesian filter.  Using our sampling strategy from the previous sections we can employ the optimal proposal distribution} 

\jonathan{we approximate importance weights, so are we doing approximate SMC or what should this be called? also in experiments we should probably run SMC with a different, e.g. sequential, importance distribution}

Given standard multi-target tracking assumptions\jonathan{independence of target motion, independence measurement error between targets, uniform distribution over measurement target associations, markov assumption, etc.}, the joint distribution over all target states $X$, measurements $Y$, and measurement-target associations $\pi$ can be factored as\footnote{For a tracking sequence of $K$ targets over $T$ time steps, $X$ is an array where row $X_t = (X_t^1, \dots, X_t^K)$ represents the state of all targets at time $t$ and element $X_t^k$ is a vector representing the state of the $k^{\text{th}}$ target at time $t$.  Likewise $Y$ is an array where row $Y_t = (Y_t^1, \dots, Y_t^K)$ represents all measurements at time $t$ and element $Y_t^k$ is a vector representing the $k^{\text{th}}$ measurement at time $t$.  Measurement-target associations are represented by the array $\pi$ where the element $\pi_t \in S_k$ is a permutations of $\{1,2,\dots,k\}$ ($S_k$ denotes the symmetric group).}

\begin{equation}
\begin{aligned}
	\Pr(X, Y, \pi) & = \Pr(X_1) \Pr(\pi_1) \Pr(Y_1 | X_1, \pi_1) \\
	&\times \prod_{t=2}^T  \Pr(X_t | X_{t-1}) \Pr(\pi_t) \Pr(Y_t | X_t, \pi_t). \\
\end{aligned}
\end{equation}

The optimal Bayesian filter for multi-target tracking is a recursive algorithm, similar to the standard Bayesian filter in the single target tracking setting, that maintains a distribution over the joint state of all targets by incorporating new measurement information as it is obtained.  It is more complex than the single target Bayesian filter because it must deal with uncertainty in measurement-target association.  As in the single target tracking setting the filter is composed of prediction and update steps. The prediction step is 

\begin{equation}
\begin{aligned}
	&\Pr(X_t | Y_{1:t-1})  \\
	= &\sum_{\pi_{1:t-1}} \Pr(X_t | Y_{1:t-1}, \pi_{1:t-1}) \Pr(\pi_{1:t-1} | Y_{1:t-1}) \\
	= &\frac{1}{k!^{t-1}}\sum_{\pi_{1:t-1}} \Pr(X_t | Y_{1:t-1}, \pi_{1:t-1}) \\
	= &\frac{1}{k!^{t-1}}\sum_{\pi_{1:t-1}} \Pr((X_t^1, \dots, X_t^K) | Y_{1:t-1}, \pi_{1:t-1}) \\
	= &\frac{1}{k!^{t-1}}\sum_{\pi_{1:t-1}} \int \dots \int \Pr(X_t^1 | X_{t-1}^1) \Pr(X_{t-1}^1 | Y_{1:t-1}, \pi_{1:t-1}) \\
	&  \times \Pr(X_t^K | X_{t-1}^K) \Pr(X_{t-1}^K | Y_{1:t-1}, \pi_{1:t-1}) dX_{t-1}^1 \dots dX_{t-1}^K\\
	= &\frac{1}{k!^{t-1}}\sum_{\pi_{1:t-1}} \int  \Pr(X_t^1 | X_{t-1}^1) \Pr(X_{t-1}^1 | Y_{1:t-1}, \pi_{1:t-1}) dX_{t-1}^1 \\
	& \times \dots \times \\
	& \int  \Pr(X_t^K | X_{t-1}^K) \Pr(X_{t-1}^K | Y_{1:t-1}, \pi_{1:t-1}) dX_{t-1}^K. \\
\end{aligned}
\end{equation}
The update step is 
\begin{equation}
\begin{aligned}
	&\Pr(X_t | Y_{1:t})  \\
	= &\sum_{\pi_{1:t}} \Pr(X_t | Y_{1:t}, \pi_{1:t}) \Pr(\pi_{1:t} | Y_{1:t}) \\
	= &\frac{1}{k!^{t}} \sum_{\pi_{1:t}} \Pr(X_t | Y_{1:t}, \pi_{1:t}) \\
	= &\frac{1}{k!^{t}} \sum_{\pi_{1:t}} \frac{\Pr(Y_t | X_t, \pi_{t}) \Pr(X_t | Y_{1:t-1}, \pi_{1:t-1})}{\int \Pr(Y_t | X_t, \pi_{t}) \Pr(X_t | Y_{1:t-1}, \pi_{1:t-1}) dX_t}  \\
\end{aligned}
\end{equation}

Unfortunately the multi-target optimal Bayesian filtering steps outlined above are computationally intractable to compute.  Even in special cases where the integrals are tractable, such as for linear Gaussian models, summation over $k!^t$ states is required.

\subsection{Sequential Monte Carlo}

Sequential Monte Carlo (SMC) or particle filtering methods can be used to sample from sequential models \jonathan{cite}.  These methods can be used to sample from the distribution defined by the optimal Bayesian multi-target filter.  When target dynamics are linear Gaussian a Rao-Blackwellized particle filter can be used to sample measurement-target associations and compute sufficient statistics for individual target distributions in closed form \cite{sarkka2004rao}.

Pseudo-code for Rao-Blackwellized sequential importance sampling is given in algorithm~\ref{alg:SIS}.  We use $KF_u(\cdot)$ and $KF_p(\cdot)$ to denote calculation of the closed form Kalman filter update and prediction steps given in equations~\ref{eq:kf_update} and \ref{eq:kf_predict} respectively.

\tiny
\begin{algorithm}\captionsetup{labelfont={sc,bf}, labelsep=newline}
  \caption{Rao-Blackwellized Sequential Importance Sampling} \label{alg:SIS}
  \jonathan{check for variable collisions in this algorithm.  should we have this pseudocode or just reference another paper?}
\begin{flushleft}

  \textbf{Outputs:} $N$ importance samples  $\pi_{1:T}^{(i)} \sim \Pr(\pi_{1:T} | Y_{1:T})$ and weights $w_T^{(i)}$ ($i \in {1, 2, \dots, n}$) with corresponding state estimates $\hat{X}_{1:T}^{(i)}$ and covariance matrices $P_{1:T}^{(i)}$.  Note $\hat{X}_{1:T}^{(i)}$ and $P_{1:T}^{(i)}$ are both arrays; $\hat{X}_t^{k(i)}$ is the $k^{\text{th}}$ target's estimated state vector at time $t$ for sample $i$.
  
\end{flushleft}
  
\begin{algorithmic}[1]
\FOR[Update particle at time $t$]{t = 1, \dots, T} 
    \FOR[Sample particle i]{i = 1, \dots, N} 
        \STATE $\pi_t^{(i)} \sim q(\pi_t | \pi_{1:t-1}^{(i)}, Y_{1:t})$
        \STATE $\pi_{1:t}^{(i)} \gets \left(\pi_{1:t-1}^{(i)}, \pi_{t}^{(i)}\right)$
        \FOR[Iterate over targets]{k = 1, \dots, K} 
            \STATE $\hat{X}_{t|t}^{k (i)}$, $P_{t|t}^{k (i)} \gets KF_u \left( \hat{X}_{t|t-1}^{k (i)}, P_{t|t-1}^{k (i)}, Y_t^{\pi(k)} \right)$ 
            \STATE $\hat{X}_{t+1|t}^{k (i)}$, $P_{t+1|t}^{k (i)} \gets KF_p \left( \hat{X}_{t|t}^{k (i)}, P_{t|t}^{k (i)} \right)$ 
            \STATE $\hat{X}_{1:t}^{(i)} \gets \left(\hat{X}_{1:t-1}^{(i)}, \hat{X}_{t}^{(i)}\right)$
            \STATE $P_{1:t}^{(i)} \gets \left(P_{1:t-1}^{(i)}, P_{t}^{(i)}\right)$
        \ENDFOR
        \STATE $w_t^{*(i)} \gets w_{t-1}^{*(i)} \frac{\prod_{k=1}^K P\left(Y_{1:t}^{\pi_t(k)} | \hat{X}_{t|t-1}^{k(i)}, P_{t|t-1}^{k(i)}\right)}{q(\pi_t | \pi_{1:t-1}^{(i)}, Y_{1:t})}$
    \ENDFOR
    \FOR[Normalize importance weights]{i = 1, \dots, N} 
        \STATE $\Tilde{w}_t^{(i)} \gets \frac{ w_t^{*(i)}}{\sum_{j=1}^N w_t^{*(j)}}$
    \ENDFOR
    \STATE \jonathan{resample if switch to SIR, otherwise mention elsewhere}
\ENDFOR

\end{algorithmic}
\end{algorithm}
\normalsize

\subsection{Optimal Proposal Distribution} \label{sec:optimal_proposal_distribution}

While SMC methods are asymptotically unbiased, their performance depends on the quality of the proposal distribution.  The optimal proposal distribution that minimizes the variance of importance weight $w_t^{*(i)}$ \citep[p.~199]{doucet2000sequential} is $q(x_t | x_{1:t-1}^{(i)}, Y_{1:t}) = \Pr(x_t | x_{t-1}^{(i)}, Y_{t})$\jonathan{where i have just abused notation with $x_t$}.  In our setting we have hidden variables $X$ and $\pi$, so we need to rewrite this as $q(X_t, \pi_t | X_{1:t-1}^{(i)}, \pi_{1:t-1}^{(i)}, Y_{1:t}) = \Pr(X_t, \pi_t | X_{t-1}^{(i)}, \pi_{t-1}^{(i)}, Y_{t}) = \Pr(X_t, \pi_t | X_{t-1}^{(i)}, Y_{t})$ (note that $X_t$ and $\pi_t$ are conditionally independent from $\pi_{t-1}^{(i)}$ given $X_{t-1}^{(i)}$).  Using Rao-Blackwellization we avoid sampling $X_t$ but instead compute sufficient statistics (mean and covariance) in closed form, so we have that the optimal proposal distribution is

 \begin{equation} \label{eq:permanent_for_tracking}
    \begin{aligned}
          & q(\pi_t | X_{1:t-1}^{(i)}, \pi_{1:t-1}^{(i)}, Y_{1:t})  \\
        = & \Pr(\pi_t | \hat{X}_{t|t-1}^{(i)}, P_{t|t-1}^{(i)}, \pi_{1:t-1}^{(i)}, Y_{1:t})  \\
        = & \Pr(\pi_t | \hat{X}_{t|t-1}^{(i)}, P_{t|t-1}^{(i)}, Y_{1:t})  \\
        = & \frac{\Pr(\pi_t, \hat{X}_{t|t-1}^{(i)}, P_{t|t-1}^{(i)}, Y_{1:t})}{\sum_{\pi_t}\Pr(\pi_t, \hat{X}_{t|t-1}^{(i)}, P_{t|t-1}^{(i)}, Y_{1:t})}  \\
        = & \frac{\Pr(Y_{1:t} | \pi_t, \hat{X}_{t|t-1}^{(i)}, P_{t|t-1}^{(i)}) \Pr(\hat{X}_{t|t-1}^{(i)}, P_{t|t-1}^{(i)} | \pi_t) \Pr(\pi_t)}{\sum_{\pi_t}\Pr(Y_{1:t} | \pi_t, \hat{X}_{t|t-1}^{(i)}, P_{t|t-1}^{(i)}) \Pr(\hat{X}_{t|t-1}^{(i)}, P_{t|t-1}^{(i)} | \pi_t) \Pr(\pi_t)}  \\
        = & \frac{\Pr(Y_{1:t} | \pi_t, \hat{X}_{t|t-1}^{(i)}, P_{t|t-1}^{(i)}) \Pr(\hat{X}_{t|t-1}^{(i)}, P_{t|t-1}^{(i)}) \Pr(\pi_t)}{\sum_{\pi_t}\Pr(Y_{1:t} | \pi_t, \hat{X}_{t|t-1}^{(i)}, P_{t|t-1}^{(i)}) \Pr(\hat{X}_{t|t-1}^{(i)}, P_{t|t-1}^{(i)}) \Pr(\pi_t)}  \\
        = & \frac{\Pr(Y_{1:t} | \pi_t, \hat{X}_{t|t-1}^{(i)}, P_{t|t-1}^{(i)}) / k!} {\sum_{\pi_t}\Pr(Y_{1:t} | \pi_t, \hat{X}_{t|t-1}^{(i)}, P_{t|t-1}^{(i)}) / k!}  \\
        = & \frac{\prod_{k=1}^K \Pr(Y_{1:t}^{\pi_t(k)} | \hat{X}_{t|t-1}^{k(i)}, P_{t|t-1}^{k(i)})} {\sum_{\pi_t}\prod_{k=1}^K \Pr(Y_{1:t}^{\pi_t(k)} | \hat{X}_{t|t-1}^{k(i)}, P_{t|t-1}^{k(i)})}.  \\
    \end{aligned}
\end{equation}
Note that the denominator of the final line in equations~\ref{eq:permanent_for_tracking} is the permanent of matrix $A$, where $(a_{jk}) = \Pr(Y_{1:t}^{j} | \hat{X}_{t|t-1}^{k(i)}, P_{t|t-1}^{k(i)})$.  Using the machinery developed throughout this paper we can sample from the optimal proposal distribution and compute approximate importance weights \jonathan{say more, how do we deal with this exactly?}.



\end{document}